\documentclass[10pt, conference]{IEEEtran}
\usepackage[utf8]{inputenc}
\usepackage{cite}
\usepackage{amsmath,amssymb,amsfonts, amsthm}
\usepackage{algorithmic}
\usepackage{graphicx}
\usepackage{textcomp}
\usepackage[caption=false]{subfig}
\usepackage[justification=centering]{caption}
\usepackage[ruled,vlined]{algorithm2e}
\usepackage{epsfig}
\usepackage{float}
\usepackage{color}
\usepackage{balance}
\usepackage{bbm}
\usepackage{textcomp}
\usepackage{enumitem}
\usepackage{url}
\usepackage{comment}
\usepackage[normalem]{ulem}
\usepackage[flushleft]{threeparttable}
\usepackage{footnote}
\makesavenoteenv{tabular}
\makesavenoteenv{table}
\def\BibTeX{{\rm B\kern-.05em{\sc i\kern-.025em b}\kern-.08em T\kern-.1667em\lower.7ex\hbox{E}\kern-.125emX}}

\newcommand{\beq}{\begin{equation}}
\newcommand{\eeq}{\end{equation}}

\newcommand{\parenthesis}[1]{\left(#1\right)}
\newcommand{\crochet}[1]{\left[#1\right]}

\newcommand{\fig}[1]{Fig. \ref{#1}}
\newcommand{\figs}[2]{Fig. \ref{#1} and \ref{#2}}

\newcommand{\mathrmbold}[1]{\boldsymbol{\mathrm{#1}}}

\newcommand{\wrt}{\textit{w.r.t.~}}
\def\softmax{\mathrmbold{softmax}}

\def\params{\mathrmbold{\theta}}
\def\cparams{\mathrmbold{w}}
\def\weights{\mathrmbold{W}}
\def\APEnergy{{\rm E}_{\rm a}(t)}
\def\UEEnergy{{\rm E}_{\rm u}(t)}
\def\ESEnergy{{\rm E}_{\rm s}(t)}
\def\TOTEnergy{{\rm E}_{\rm tot}(t)}
\def\WTOTEnergy{{\rm E}_{\rm w}(t)}

\def\obsRadio#1{\mathrmbold{o}_{#1}^{\textsc{R}}(t)}
\def\obsComp#1{\mathrmbold{o}_{#1}^{\textsc{C}}(t)}

\def\maxof#1#2{\mathrm{max}\left(#1, #2\right)}
\def\minof#1#2{\mathrm{min}\left(#1, #2\right)}

\newcommand\blfootnote[1]{%
  \begingroup
  \renewcommand\thefootnote{}\footnote{#1}%
  \addtocounter{footnote}{-1}%
  \endgroup
}

\floatstyle{ruled}
\newfloat{algorithm}{tbp}{loa}
\providecommand{\algorithmname}{Algorithm}
\floatname{algorithm}{\protect\algorithmname}

\newtheorem{remark}{Remark}
\newtheorem{proposition}{Proposition}

\newcommand{\titleheader}{This work has been accepted for publication in IEEE ICC'2021 Workshop - DAWNZ}

\makeatletter
\def\ps@headings{%
\def\@oddhead{\mbox{}\scriptsize \titleheader}
\def\@oddfoot{\scriptsize \@date \hfil }%
}

\def\ps@IEEEtitlepagestyle{%
\def\@oddhead{\mbox{}\scriptsize \titleheader \rightmark \hfil }%
\def\@oddfoot{\scriptsize \@date \hfil }%
}
\makeatother

\title{Energy Efficient Edge Computing: When Lyapunov Meets Distributed Reinforcement Learning}
\author{\IEEEauthorblockN{Mohamed Sana$^{1}$, Mattia~Merluzzi$^{2}$, Nicola di Pietro$^{3}$, Emilio Calvanese Strinati$^{1}$}

\IEEEauthorblockA{$^{1}$CEA-Leti, Université Grenoble Alpes, F-38000 Grenoble, France\\
$^{2}$Sapienza Univ. of Rome, DIET, via Eudossiana 18, 00184, Rome, Italy\\
$^{3}$Athonet, via Cà del Luogo 6/8, 36050, Bolzano Vicentino (VI), Italy\\
Email : \{mohamed.sana, emilio.calvanese-strinati\}@cea.fr, mattia.merluzzi@uniroma1.it, nicola.dipietro@athonet.com}}

\begin{document}
\maketitle\blfootnote{This work was supported by the H2020 EU/Taiwan Project 5G CONNI, Nr. 861459.}
\begin{abstract}
In this work, we study the problem of energy-efficient computation offloading enabled by edge computing. In the considered scenario, multiple users simultaneously compete for limited radio and edge computing resources to get offloaded tasks processed under a delay constraint, with the possibility of exploiting low power sleep modes at all network nodes.
The radio resource allocation takes into account inter- and intra-cell interference, and the duty cycles of the radio and computing equipment have to be jointly optimized to minimize the overall energy consumption. To address this issue, we formulate the underlying problem as a dynamic long-term optimization. Then, based on Lyapunov stochastic optimization tools, we decouple the formulated problem into a CPU scheduling problem and a radio resource allocation problem to be solved in a per-slot basis. Whereas the first one can be optimally and efficiently solved using a fast iterative algorithm, the second one is solved using \emph{distributed} multi-agent reinforcement learning due to its non-convexity and NP-hardness. The resulting framework achieves up to $96.5\%$ performance of the optimal strategy based on exhaustive search, while drastically reducing complexity. The proposed solution also allows to increase the network's energy efficiency compared to a benchmark heuristic approach. %energy consumption by $10\%$ compared to heuristic approaches.
\end{abstract}

%\begin{IEEEkeywords}
%Mobile Edge Computing, 5G and Beyond, Green Networking, Multi-agent Reinforcement Learning.
%\end{IEEEkeywords}

%\vspace{-0.1cm}
\section{Introduction}
Wireless communication networks are experiencing an unprecedented revolution, evolving from pure communication systems towards a tight integration of communication, computation, caching, and control \cite{6Gstrinati}. %In particular, 5G (and beyond) networks are expected to enable a plethora of new services, which are no more developed only for mobile end users, but for different sectors (\textit{verticals}), such as automotive, healthcare, Industry 4.0 and smart cities.  
Such a heterogeneous ecosystem requires a flexible network design and orchestration, able to accommodate on the same network infrastructure all the different services with their requirements in terms of energy, latency, and reliability. This requires an enhancement of the radio access network, e.g., through the adoption of millimeter wave (mmWave) communications, densification of access points (APs), and a flexible management of the physical layer \cite{ahmadi20195g}. In addition, the deployment of computing and storage capabilities at the network edge will enable network function virtualization and a fast %computation and 
processing of the myriad of data collected %from the environment 
by sensors, cars, mobile devices, etc. For this, Edge Computing\footnote{It is standardized by ETSI as Multi-access Edge Computing (MEC) \; {(\url{https://www.etsi.org/technologies/multi-access-edge-computing})}.} %is introduced
was conceived to enable energy efficient, low-latency, highly reliable services by bringing cloud resources close to the users.
%(\textcolor{blue}{Emilio: which is the intended message? we cannot do everything? we are limited in resources and we need to optimize to ACHIEVE THE SAME TARGET Goal? one important point is the effectiveness of the solution, that means that given the limited but local resources we design the services differently and we allocate and free resources as needed} albeit in a limited manner) close to the end users. %However, since MEC resources are limited compared to the central cloud, an effective joint management of radio and computation resources is one of the fundamental challenges of future mobile networks. 
%As a further aspect of this complex ecosystem, it should be noted that, with respect to the past

In this context, \emph{dynamic computation offloading} %is a service that 
allows resource-poor devices to %reliably 
transfer application execution to %nearby 
Edge Servers (ESs) %, with the purpose of reducing 
to reduce 
energy consumption and/or latency. From a network management %point of view
perspective, this task is complex and requires the joint optimization of radio and computation resources. %In the literature, 
This problem has received wide attention\cite{Pham2019Survey}. In %particular
\cite{LiGuan19}, a scheduling strategy is proposed to %find a trade off between 
counterbalance task completion ratio and throughput, hinging on Lyapunov optimization. \cite{ChenZhou17} aims at minimizing the long-term average delay under a long-term average power consumption constraint. In \cite{Wang19}, the long-term average energy consumption of a MEC network is minimized under a delay constraint, using a MEC sleep control. %Also, in~
\cite{ChangMiao18-2} %the problem is formulated as the minimization of 
minimizes the energy consumption under a mean service delay constraint, optimizing the number of active base stations and the ESs' computation resource allocation, % at the ES, %while considering 
leveraging sleep modes for APs and ESs. In \cite{Nan17}, Lyapunov optimization is used to reduce the energy consumption of a fog network, %while 
guaranteeing an average response time. %In \cite{Yu_Pu_2018}, %the authors exploit 
%Lyapunov optimization, Lagrange multipliers, and sub-gradient techniques are exploited to optimize devices' and APs' energy consumption under latency constraints, with AP sleep states.
In \cite{Yu_Pu_2018}, %the authors exploit 
devices' and APs' energy consumption is minimized via Lyapunov optimization, Lagrange multipliers and sub-gradient techniques, using APs' sleep state, %are exploited to optimize  
under latency constraints.

Recently, the advances of machine learning and Deep Reinforcement Learning (DRL) in wireless networks have opened up new possibilities for %designing
low-complexity and efficient algorithms for MEC \cite{6Gstrinati}, %. This is of interest, 
especially when model-based optimization is %becoming more and more 
challenged by the difficulty or even impossibility of writing mathematical models that accurately predict the networks' behavior. In this sense, the authors of \cite{Bi2020} propose to couple %exploit 
model-based Lyapunov optimization %coupled 
with model-free DRL and %. In particular, they
formulate a sum-rate maximization problem under queue stability and long-term device energy constraints. However, %only the energy consumption of static users is taken into account. Moreover, 
their reference scenario considers a single AP, and no CPU scheduling is optimized at the ES. %Then, they jointly optimize the offloading decision, the time dedicated to each user for offloading over the radio interface (assuming a time division multiple access), the local CPU frequency at each device, and the energy spent for transmission. Only the energy consumption of the users is taken into account, the reference scenario considers a single AP, and no CPU scheduling is optimized at the ES.
\cite{Bae2020} also considers the same approach, with the aim of minimizing the sum of the power consumption of the edge nodes, and a cost charged by a central cloud to help the edge node in processing the tasks under stability constraints. However, they do not consider the energy consumption of end users and APs.  %The optimization variables are the CPU scheduling at the ES and the bandwidth allocation to send tasks from the ES to the central cloud. Then, no radio parameters between end devices and AP are considered, thus neglecting the energy consumption of end devices and AP.

In this paper, we combine the convenience of a model-based solution that exploits Lyapunov stochastic optimization, with the power of model-free solutions based on multi-agent reinforcement learning (MARL), aiming at energy efficient computation offloading from an overall network perspective. We consider %a scenario where 
multiple user equipments (UEs) that perform computation offloading and compete for computation resources %in a single core of the ES
at an ES and for radio resources, %generating interference onto 
interfering onto each other's transmissions. %In particular, 
We %formulate a dynamic %offloading 
treat the problem as a long-term system energy minimization with average end-to-end delay constraints, in a network with multiple APs and one ES, all %capable of 
exploiting low-power sleep operation modes.  
Although we do not assume any knowledge on radio channels and data arrival statistics, %we propose an 
our online solution %that, in each time slot, 
optimizes in each time slot the UE-AP association in a distributed way, and the ES's CPU scheduling via a fast iterative algorithm whose solution's complexity scales linearly in the number of UEs. 

Compared to the cited works, the originality of our strategy %lies in the capability of 
consists in \emph{simultaneously}: $i)$ minimizing the duty cycles of all the network elements under delay constraints; $ii)$ effectively managing radio interference; $iii)$ being low-complexity; $iv)$ combining Lyapunov optimization with DRL; $v)$ being distributed and compatible with UE's mobility. The latter point, in sharp contrast with \cite{Bi2020}, results from the \emph{zero-shot generalization} capability of our solution: it optimizes the learned computation offloading policy for all possible deployments of UEs \textit{using attention neural networks}, and adapts when the number of UEs differs from the initial training point.

\section{System Model}\label{sec:system_model}
%\vspace{-0.15cm}

We consider %a network
the scenario %as depicted in
of~\fig{fig:network-model}, where $K$ UEs offload %their own applications 
computational tasks to an ES, via %by accessing
one out of $N$ possible APs. Let $\mathcal{U}$ and $\mathcal{A}$ be the sets of UEs and APs, respectively. Also, let $\mathcal{A}_k$ be the set of APs UE $k$ can be associated with. %Since we deal with a 
In our dynamic system, %we model
time is divided in slots of equal duration $\tau$. Specifically, we assume that a fraction $\beta\in(0,1)$ of each slot is devoted to control signaling and ($1-\beta$) to data transmission and computation, both potentially happening simultaneously. At each time slot $t$, radio channels and data arrivals at the UEs vary according to \emph{a priori} unknown statistics. Consequently, the achievable data rate over the radio channels and the computation rate at the ES vary with time. %depend on a certain policy based are chosen depending on instantaneous observations. \textcolor{blue}{Comment on this}
\vspace{-0.14 cm}
\subsection{Radio access and data rate model}
In this paper, we consider uplink communications for computation offloading. %For %the %multiple access over the 
%radio channels, 
Specifically, we assume %a
spatial division multiple access. %system, in which 
All UEs are served by the APs over the same time-frequency resources but with different beams. In this scenario, %a UE's %experiences a data rate $R_k(t)$, 
%achievable rate %$R_k(t)$ (in bits/s)
uplink communications are affected by both intra- and inter-cell interference.
%Since we consider a multi-AP scenario, 
%We introduce the binary variable $x_{kn}(t)$, which equals $1$ if UE $k$ is assigned to AP $n$ at time $t$, and $0$ otherwise. %In particular,
Indeed, suppose that UE $k$ is served by AP $n$ at time $t$. Let $p_k^{u,\rm tx}(t)$ be the uplink transmit power of UE $k$, $G_{kn}^{\rm ch}(t)$ the channel power gain between UE $k$ and AP $n$, $G_{kn}^{\rm tx}(t)$ the transmit antenna gain towards the direction of AP $n$, $G_{kn}^{\rm rx}(t)$ the receive antenna gain, $B$ the allocated bandwidth, and $N_0$ the noise power spectral density. %\textcolor{green}{Note that another generic UE $k'$, even if not associated with AP $n$, interferes with UE $k$, due to a non zero antenna gain $G_{k'n}^{\rm tx}(t)$ towards the direction of AP $n$.}
Then, the signal-to-interference-plus-noise ratio (SINR) is given by $$\textrm{SINR}_{k}(t)=\frac{%I_k^u(t)
p_k^{u,\rm tx}(t)G_{kn}^{\rm tx}(t)G_{kn}^{\rm ch}(t)G_{kn}^{\rm rx}(t)}{\mathcal{I}_{kn}(t)+N_0B},$$ where %$\displaystyle I_k^{\rm u}(t) = \max_{n \in \mathcal{A}_k}\{x_{k,n}(t)\}$, and 
%the interference $\mathcal{I}_{kn}(t)$ is 
$\mathcal{I}_{kn}(t)=\sum\nolimits_{k'\in\mathcal{U}\setminus\{k\} }%I_{k'}^u(t)
p_{k'}^{u,\rm tx}(t)G_{k'n}^{\rm tx}(t)G_{k'n}^{\rm ch}(t)G_{k'n}^{\rm rx}(t)$
is the overall interference. Then, the achievable rate of UE $k$ at time $t$ is %given by the Shannon formula as
$R_k(t)=%\sum_{n}
B\log_2(1+\textrm{SINR}_{k}(t))$.
If the $k$-th UE's offloadable data unit is encoded into $S_k$ bits, the number of data units transmitted in the uplink direction at time $t$ is
\begin{equation}\label{num_uplink}
    N_k^u(t)=\left\lfloor\frac{(1-\beta)\tau R_k(t)}{S_k}\right\rfloor.
\end{equation}
%\vspace{-0.5cm}
\subsection{Computation model}
At the ES, %we assume that 
all UEs are served by one core and compete for the CPU time in each time slot. %In particular, 
Given a %CPU 
core frequency $f_c(t)$ (measured in CPU cycles per second), each UE is allocated a portion $f_k(t)$ of $f_c(t)$, with $\sum_{k=1}^Kf_k(t)\leq f_c(t)$. %Also, we assume that a certain number of CPU cycles is needed to process one data unit. 
Then, denoting by $J_k$ the number of processed data units per CPU cycle, the number of data units processed over one slot is
\begin{equation}\label{num_computed}
    N_k^c(t)=\lfloor(1-\beta)\tau f_k(t)J_k\rfloor.
\end{equation}

\begin{figure}
    \centering
    \includegraphics[width=\columnwidth]{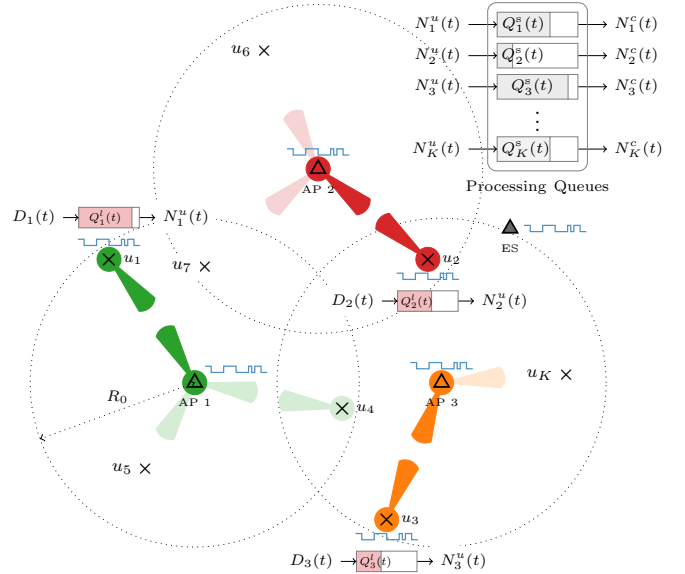}
    \caption{Network model with 3 APs deployed with $K$ UEs.}
    \label{fig:network-model}
\end{figure}

\vspace{-0.2cm}
\subsection{Delay and queuing model}
In our setting, computation offloading involves two steps: $i)$ an uplink transmission phase of input data from the UEs; $ii)$ a computation phase at the ES.
New data units are continuously generated from an application at the UE's side and consequently offloaded and processed at the ES. In particular, once generated, data are queued locally at the UEs, then uploaded to the ES through one AP with time varying data rate as in \eqref{num_uplink}. At the ES, they are queued waiting to be processed and they are finally computed with time varying computational rate as in \eqref{num_computed}. Thus, we represent the overall system through a simple queuing model involving both queues, %This model allows us to characterize the total average delay experienced by a data unit from its generation at the mobile side until the computation is performed by the ES. 
%The model is 
synthetically depicted in \fig{fig:network-model}. Accordingly, each data unit experiences two different delays: $i)$ a communication delay, including buffering at the UE; $ii)$ a computation delay, including buffering at the ES. As shown later, we take into account these two delays jointly, as in \cite{Merluzzi2020URLLC}.
%\underline{\textit{Uplink communication queue:}} 
UE $k$'s uplink communication queue evolves as
%\begin{equation}\label{num_bits}
 %   N_k^{u}(t)=\left\lfloor\frac{\tau R_k(t)}{S_k}\right\rfloor,\nonumber
%\end{equation}
%where $S_k$ is the size in bits of a data unit. The local queue evolution of data can be written as
%\begin{equation*}%\label{q_loc}
$Q_k^l(t+1)=\max\left(0, Q_k^l(t)-N_k^u(t)\right)+D_k(t)$,
%\end{equation*}
where $D_k(t)$ is the number of newly arrived offloadable data units generated by the application %that runs 
at the UE at time $t$, %, and $N_k^u(t)$ reads as in \eqref{num_uplink}. 
%$D_k(t)$ is modeled as 
which is the realization of a random process whose statistics are %supposed to be 
unknown \emph{a priori}. %\underline{\textit{Remote computation queue:}} 
The remote computation queue at the ES evolves as
%In this work, we assume that the number of input data units treated by the ES for UE $k$ is proportional to the number of CPU cycles allocated for this task. Thus, if we denote by $J_k$ the proportionality factor, we can write number of computed tasks during time slot $t$ as
%\begin{equation}
 %   N_k^c(t)=\lfloor\tau f_k(t)J_k\rfloor
%\end{equation}
%where $f_k(t)$ is the CPU cycle frequency assigned to user $k$ during time slot $t$.
%can define a queue of input data units waiting for being processed at the server, which evolves as follows:
%\begin{align*}%\label{q_rem}
$Q_k^{\rm s}(t+1)=\max(0,Q_k^{\rm s}(t)-N_k^c(t))+\min(Q_k^l(t),N_k^u(t)).$
%\end{align*}
%where $N_k^c(t)$ reads as in \eqref{num_computed}.

\vspace{0.1cm}
\noindent
\textbf{End-to-end delay constraints.} %
By Little's law \cite{little1961}, given a stationary queueing system, the average overall service delay is proportional to the average queue length. Then, in our system, the overall delay is directly related to the sum of the uplink communication queue and the computation queue $Q_k^\text{tot}(t)=Q_k^l(t)+Q_k^{\rm s}(t)$. In particular,
if $\displaystyle\bar{D}_k=\mathbb{E}\left\{D_k(t)/\tau\right\}$ is the average data unit arrival rate, the long-term average end-to-end delay $L_k^{\rm avg}$ experienced by a data unit generated by UE $k$ is given by the ratio between the average of %the total queue length 
$Q_k^{\rm tot}$ and the average arrival rate.
%\begin{equation}\label{latency}
%\bar{D}_k^\infty=\lim_{T\to\infty}\frac{1}{T}\sum_{t=1}^T\mathbb{E}\left\{\frac{Q_k^{\rm tot}(t)}{\bar{D}_k}\right\},
%\end{equation}
%where the expectation is taken with respect to the random radio channel and data arrival realizations. 
Thus, our %first aim is to guarantee a 
constraint on the long-term average delay $L_k^{\textrm {avg}}$ is formulated as follows:
\begin{equation}\label{avg_q}
\lim_{T\to\infty}\frac{1}{T}\sum_{t=1}^T\mathbb{E}\left\{Q_k^{\rm tot}(t)\right\}\leq Q_k^{\textrm{avg}} = L_k^{\textrm {avg}}\bar{D}_k,\quad \forall k.
\end{equation}
Note: $\bar{D}_k$ is not known \emph{a priori}, but can be estimated online.
%\vspace{-0.1cm}
\subsection{Energy consumption model}
%\textcolor{red}{As already mentioned, in each slot, we assume that a fraction $\beta\tau$ is dedicated to the control signaling, and remain part, i.e. $(1-\beta)\tau$ is dedicated to the actual computation offloading (data uplink and downlink transmission and elaboration at the ES). Thus, since we exploit low power sleep modes to reduce energy consumption, each network agent will be anyway active for the fraction $\beta\tau$ for control signaling, thus consuming the active power for at least $\beta\tau$ seconds in each time slot.}
%In this paper, 
We exploit low-power operation modes at UEs, APs, and ES %, as a way 
to reduce the overall energy consumption. However, due to control signaling, UEs, APs, and ES are active for at least $\beta \tau$ seconds in each slot, consuming active power. Hence, the energy consumption of each entity %can be 
is modeled as follows:

\vspace{0.1cm}
\noindent
\textbf{UE Energy Consumption.} %
Let $x_{k,n}(t) \in \{0,1\}$ be an association variable %such that $x_{kn}(t)=1$ 
equal to $1$ if and only if UE $k$ offloads its data via AP $n$ at time $t$. %, and $x_{kn}(t)=0$ otherwise. Also, 
Let $p_{k}^{\rm u, off}$ and $p_{k}^{\rm u, on}$ %, $p_{k}^{\rm u, tx}(t)$ 
be UE $k$'s sleep and active power, respectively. %, active power, and uplink transmission power of UE $k$. 
Then, the total UEs' energy consumption at time $t$ is:
\begin{IEEEeqnarray}{rCl}\label{UE_energy}
{\rm E}_{\rm u}(t) &=& \sum_{k=1}^K\tau\bigg[(1-\beta) \bigg(I_k^{\rm u}(t) \parenthesis{p_k^{\rm u, on}+p_k^{\rm u, tx}(t)}\nonumber\\
&&\qquad\qquad+ (1-I_k^{\rm u}(t))p_k^{\rm u, off}\bigg)+\beta p_k^{\rm u, on}\bigg],
\end{IEEEeqnarray}
where $\displaystyle I_k^{\rm u}(t) = \max_{n \in \mathcal{A}_k}\{x_{k,n}(t)\}$ indicates if UE $k$ is active: %or not. 
%Indeed, %in a given time slot, 
%a UE can decide to not associate with any AP, hence, 
if it %UE $k$ 
does not transmit at time $t$, $I_k^{\rm u}(t) =0$, and $p_k^{\rm u, tx}(t)=0$.

\vspace{0.1cm} %Hence, the total UE transmit energy is simply given by
%\begin{equation}\label{UE_energy}
 %   \UEEnergy = \sum_{k=1}^K {\rm E}_k^{\rm u}(t).
%\end{equation}
\noindent
\textbf{AP Energy Consumption.} %
%Denoting by %us denote with 
Let $p_{n}^{\rm a, off}$ and $p_{n}^{\rm a, on}$ be the $n$-\textit{th} AP's sleep and active power consumption, respectively. %, respectively, the
The total APs' energy consumption at time $t$ is
\begin{IEEEeqnarray}{rCl}\label{AP_energy}
{\rm E}_{\rm a}(t) &=&\sum_{n=1}^N\tau\bigg[ (1-\beta) \parenthesis{I_n^{\rm a}(t) p_n^{\rm a, on} + (1-I_n^{\rm a}(t))p_n^{\rm a, off}}\nonumber\\
&&\qquad\qquad+\beta p_n^{\rm a, on}\bigg],
\end{IEEEeqnarray}
where $\displaystyle I_n^{\rm a}(t) = \max_{k \in \mathcal{U}}\{x_{k,n}(t)\}$ indicates whether AP $n$ is active ($I_n^{\rm a}(t) = 1$) or not ($I_n^{\rm a}(t) =0$). %, i.e an AP serving at least one UE is active. %Hence, the total AP power consumption is:
%\begin{equation}\label{AP_energy}
 %  \APEnergy= \sum_{n=1}^N {\rm E}_n^{\rm a}(t)
%end{equation}

\vspace{0.1cm}
\noindent
\textbf{ES Energy Consumption.} %
%In this paper, we adopt C-states operating on a specific core, dedicated to treat the offloaded tasks of all the mobile users of our system. In particular, we consider the C0-state, in which the CPU core is active and executing some thread, and the C1-state, in which the CPU clock frequency is driven to zero. 
%In this paper, we 
To reduce the energy consumption, we adopt both a low-power sleep mode for the ES, when no computation is performed at a given slot $t$, %time slot, 
and a scaling of the CPU frequency $f_c(t)$, %to reduce the energy consumption 
when the CPU is active and computing \cite{LeSueur10}. %Then, in our model, 
Namely, the CPU core consumes a power $p_{\rm s}^{\rm on}$ %just for being 
in active state, and a power $p_{\rm s}^{\rm off} < p_{\rm s}^{\rm on}$ in sleep state. %Moreover, 
When the ES is active, the dynamic power spent for computation is 
%\begin{equation}\label{eq:p_comp}
%p_{\rm s}^{\rm c}(t)=\kappa f_c^3(t),
%\end{equation}
$p_{\rm s}^{\rm c}(t)=\kappa f_c^3(t)$, where %$f_c(t)$ is the CPU frequency used during time slot $t$, and 
$\kappa$ is the effective switched capacitance of the processor \cite{Burd1996}. %We suppose that it is possible to use dynamic voltage frequency scaling to scale down the frequency \cite{LeSueur10}, thus reducing the dynamic power consumption. 
In particular, we assume that $f_c(t)$ can be dynamically selected from a finite set $\mathcal{F}=\{0,\ldots,f_{\max}\}$. Therefore, the %overall system 
ES's energy consumption at time $t$ is
\begin{IEEEeqnarray}{rCl}\label{MEH_energy}
\ESEnergy &=&(1-\beta)\tau \parenthesis{I_{\rm s}(t) \parenthesis{p_{\rm s}^{\rm on} + p_{\rm s}^{\rm c}(t)} + (1-I_{\rm s}(t))p_{\rm s}^{\rm off}} \nonumber\\
&&{~}+ \beta\tau p_{\rm s}^{\rm on},
\end{IEEEeqnarray}
where $I_{\rm s}(t)=\mathbf{1}\{f_c(t)\}$, with $\mathbf{1}\{\cdot\}$ the indicator function, 
indicates whether the ES is active ($I_{\rm s}(t)=1$) or not ($I_{\rm s}(t)=0$). %Then, 
From (\ref{UE_energy}), (\ref{AP_energy}), (\ref{MEH_energy}), the total system energy consumption at time $t$ is $\TOTEnergy = \ESEnergy + \APEnergy + \UEEnergy$. %In this paper,
Our objective function %, we consider 
is a convex combination of %\eqref{UE_energy}, \eqref{AP_energy}, and \eqref{MEH_energy}:
its terms:
%the energy consumption of all elements: %, i.e. %written as 
\begin{equation}\label{weighted_energy}
    \WTOTEnergy = \alpha_1 \UEEnergy + \alpha_2 \APEnergy + \alpha_3 \ESEnergy,
\end{equation}
with $\sum_{i=1}^3\alpha_i=1$. %In particular, 
Different $\alpha_i$ lead to different strategies. For example, $\alpha_1=1$ models a \textit{user-centric} strategy, where only UEs' energy consumption is optimized. $\alpha_i=1/3, \forall i$ yields a \textit{holistic} strategy that includes the whole network's energy~\cite{MerluzziDMEC}. %, where the whole network energy is taken into account.
%\vspace{-0.3cm}
\section{Problem Formulation}
%\vspace{-0.4cm}
%Based on the system model %defined in 
%of Section \ref{sec:system_model}, we formulate 
We formulate the following minimization problem %to minimize the long-term average sum of the network entities' 
on the weighted network energy consumption, %$\WTOTEnergy$ defined in \eqref{weighted_energy}, %under average delay constraints 
subject to \eqref{avg_q} and instantaneous constraints on the optimization variables:
\vspace{-0.3cm}
\begin{align}
\underset{\{\mathbf{\Psi}(t)\}}{\mathrm{minimize}}~~& \lim_{T\to\infty}\frac{1}{T}\sum_{t=1}^{T}\mathbb{E}\{\WTOTEnergy\}, \tag{$\mathcal{P}_0$} \label{pb-genesis}\\
\mathrm{subject~to~}{}& %\lim_{T\to\infty}\frac{1}{T}\sum_{t=1}^T\mathbb{E}\crochet{Q_k^{\textrm{tot}}(t)}\leq Q_k^{\textrm {avg}}, && \forall k; \tag{C1}
\;\text{Eqn.}\; \eqref{avg_q} \tag{C1}\label{eq:C1}\\
     {}& x_{k,n}(t)\in\{0,1\}, && \forall k,n, \tag{C2}\label{eq:C2}\\
	 {}&\sum_{k \in \mathcal{U}}x_{k,n}(t) \leq N_{n}, &&\forall n,\tag{C3}\label{eq:C3}\\
	 {}&\sum_{n \in \mathcal{A}_k}x_{k,n}(t) \leq 1, && \forall k,\tag{C4}\label{eq:C4}\\
	% {}& I_{\rm s}(t) \in \{0,1\}, \tag{C5}\label{eq:C5}\\
	 {}&f_{k}(t)\geq 0, && \forall k,t,\tag{C5} \label{eq:C5}\\
	  {}&f_{c}(t) \in \mathcal{F}, && \forall t, \tag{C6} \label{eq:C6}\\
	 {}&\sum_{k \in \mathcal{U}} f_{k}(t)\leq f_c(t), && \forall t, \tag{C7} \label{eq:C7}
\end{align}
where $\mathbf{\Psi}(t)=\crochet{\{x_{k,n}(t)\}_{k,n},f_c(t),\{f_k(t)\}_k}$ and the expectation is taken with respect to the random input data unit generation %arrivals 
and radio channels, whose statistics are unknown. %supposed to be unknown in advance.
\eqref{eq:C1} %of \eqref{pb-genesis} 
is the delay constraint. %whereas  Constraints
\eqref{eq:C2}-\eqref{eq:C7} mean what follows: %have the following meaning: 
\eqref{eq:C2} the UE-AP association variables are binary; \eqref{eq:C3} the number of UEs assigned to each AP cannot exceed a maximum $N_n$; \eqref{eq:C4} each UE is assigned to at most one AP; %\eqref{eq:C5} the active state variable of the ES is binary; 
\eqref{eq:C5}-\eqref{eq:C7} the computation frequencies assigned to each user are non negative and their sum cannot exceed the total CPU frequency of the ES, chosen from the finite set $\mathcal{F}$. %\eqref{eq:C6} the total CPU frequency of the ES is chosen from a discrete finite set $\mathcal{F}$; \eqref{eq:C7} the sum of the CPU resources assigned to all UEs cannot exceed $f_c(t)$.

Directly solving \eqref{pb-genesis} is very challenging due to the unavailability of the statistics. Therefore, we hinge on %the tools of 
Lyapunov optimization. %First of all, to 
To handle \eqref{eq:C1}, we introduce \textit{virtual queues} \cite{Neely10} $Z_k(t)$ that evolve as
%\begin{equation*}%\label{virtual_Z}
$Z_k(t+1)=\max(0,Z_k(t)+Q_k^{\rm tot}(t+1)-Q_k^{\rm avg})$
%\end{equation*}
for all $k$.
$Z_k(t)$ is a state variable that measures how the system behaves w.r.t. \eqref{eq:C1}: it increases if the instantaneous value of $Q_k^{\rm tot}(t)$ violates the constraint, and decreases otherwise. It can be easily shown \cite{Neely10} that \eqref{eq:C1} is guaranteed if $Z_k(t)$ is \textit{mean rate stable}, \emph{i.e.}, $\lim_{T\to\infty}\frac{\mathbb{E}\{Z_k(T)\}}{T}=0$. %Now, to ensure the mean rate stability of $Z_k(t)$, 
To ensure this, we introduce the \emph{Lyapunov} %function 
and the \textit{drift-plus-penalty} functions:
\vspace{-0.3cm}
\begin{align}\label{drift_plus_penalty}
 L(\mathbf{Z}(t)) &=\frac{1}{2}\sum_{k=1}^K Z_k(t)^2,\\
    \Delta_p(\mathbf{Z}(t))  &=\mathbb{E}\left\{L(\mathbf{Z}(t+1))-L(\mathbf{Z}(t))+\Omega\cdot\WTOTEnergy|\mathbf{Z}(t)\right\}.\nonumber
\end{align}
Here, $L(\mathbf{Z}(t))$ %represents a measure of 
``measures'' the %overall
system's congestion, whereas %the drift-plus-penalty function
$\Delta_p(\mathbf{Z}(t))$ is the conditional expected variation of $L(\mathbf{Z}(t))$ over one slot, plus a penalty factor %with a weighting parameter 
weighted by $\Omega$, which trades off queue backlogs and the objective function of \eqref{pb-genesis} \cite{Neely10}. %The higher is $\Omega$, the more importance is assigned to the objective function rather than to queue backlogs. 

%Also, as $\Omega$ increases, the optimal solution of \eqref{pb-genesis} is asymptotically achieved \cite{Neely10}. 
%However, directly minimizing \eqref{lyapunov_drift} does not take into account the objective function of \eqref{pb-genesis}, so that it can lead to an unnecessary system energy consumption. However, to guarantee the mean rate stability of the virtual queues while reducing the value of the objective function, we introduce the \textit{drift-plus-penalty} %function
%\begin{equation}\label{drift_plus_penalty}
 %   \Delta_p(\mathbf{Z}(t)) =\Delta(\mathbf{Z}(t))+\mathbb{E}\left\{\Omega\cdot\WTOTEnergy|\mathbf{Z}(t)\right\},
%\end{equation}
%which is a penalized version of \eqref{lyapunov_drift}, with a weighting parameter $\Omega$ used ti assign more importance to the objective function or queue lengths. 
%Now, rather than directly minimizing \eqref{drift_plus_penalty}, 

\begin{proposition}
If the radio channel states and the input data generation are i.i.d.~over time slots, the \emph{optimal} solution of \eqref{pb-genesis} is obtained when optimally solving the following two sub-problems for a sufficiently high value of $\Omega$.

\noindent
\textbf{%Sub-problem 
For CPU scheduling:} %For CPU scheduling, 
At time $t$, solve the following %optimization 
problem:
\begin{align}\label{sub-pb-computation}
\underset{\{f_c(t), \{f_k(t)\}_{k}\}}{\mathrm{minimize}}\;\, & G_1(t) = \Omega \alpha_3 \ESEnergy + \sum_{k=1}^K\bigg[ -2Q_k^{\rm s}(t)\tau f_k(t)J_k  \nonumber\\
&\hspace{-0.35 cm}+ \max\left(0,Q_k^{\rm s}(t)-\tau f_k(t)J_k+1\right)Z_k(t) \bigg]\tag{$\mathcal{P}_1$} \\
\mathrm{subject~to~}{}&\quad \eqref{eq:C5}\text{-}\eqref{eq:C7}\;\text{of}\;\eqref{pb-genesis}\nonumber.
\end{align}

\noindent 
\textbf{%Sub-problem 
For UE-AP association:} %For the UE-AP association 
At time $t$, solve the following: %optimization problem:
\begin{align}\label{sub-pb-radio}
\underset{\{x_{kn}(t)\}_{k,n}}{\mathrm{minimize}} \quad & G_2(t) = \Omega\cdot \left(\alpha_1 \UEEnergy+\alpha_2 \APEnergy \right)\nonumber\\
&+\sum_{k=1}^K\bigg[\left(-\frac{3}{2}Q_k^l(t) + Q_k^{\rm s}(t)\right)N_k^{u}(t) \nonumber\\
&+\max\left(0,Q_k^l(t)-N_k^{u}(t)\right)Z_k(t)\bigg] \tag{$\mathcal{P}_2$}\\
\mathrm{subject~to~}{}&\quad \eqref{eq:C2}\text{-}\eqref{eq:C4} \;\text{of}\; \eqref{pb-genesis}. \nonumber %\quad & x_{k,n}(t) \in \{0,1\}, \quad\forall k,n\nonumber\\
	 %{}\quad &\sum_{k \in \mathcal{U}_n}x_{k,n}(t) \leq N_{n}, \quad \forall n \nonumber\\
	 %{}\quad &\sum_{n \in \mathcal{A}_k}x_{k,n}(t) \leq 1, \quad \forall k \nonumber
\end{align}
\end{proposition}

\begin{proof}[Sketch of proof]
From \cite{Neely10}, we know that, if %the drift-plus-penalty function 
$\Delta_p(\mathbf{Z}(t))$ %in \eqref{drift_plus_penalty} 
is bounded, the $Z_k(t)$'s are mean rate stable and therefore \eqref{eq:C1} is guaranteed. Now, %it can be formally proved 
the main statement follows because minimizing a proper upper bound of $\Delta_p(\mathbf{Z}(t))$ under \eqref{eq:C2}-\eqref{eq:C7} %of \eqref{pb-genesis} 
is  equivalent to optimally solving \eqref{pb-genesis} when $\Omega$ is sufficiently large \cite[Th. 4.8]{Neely10}. %Therefore, our policy proceeds by minimizing %an upper bound of \eqref{drift_plus_penalty} to ``push'' the queues towards lower congestion states, i.e. towards system stability. %In particular, 
%It can be shown that \eqref{drift_plus_penalty} enjoys 
The considered upper bound is:
\begin{IEEEeqnarray}{ll}\label{eq:upperbound_dpp}
    {}&\Delta_p(\mathbf{Z}(t)) \leq \zeta +\mathbb{E}\bigg\{\sum_{k=1}^K\bigg[\chi_k(t)-2Q_k^{\rm s}(t)\tau f_k(t)J_k \nonumber\\
    {}&+\!\left(\maxof{0}{Q_k^{\rm s}(t)-N_k^c(t)}+\maxof{0}{Q_k^l(t)-N_k^{u}(t)}\right)\!Z_k(t)\nonumber\\
    {}&+\parenthesis{-\frac{3}{2}Q_k^l(t) + Q_k^{\rm s}(t)}N_k^{u}(t) \bigg]+\Omega \cdot\WTOTEnergy\bigg|\mathbf{Z}(t) \bigg\},
\end{IEEEeqnarray}
where $\zeta>0$ is a constant and $\chi_k(t)$ does not depend on the optimization variables. %\textcolor{red}{The derivations leading to \eqref{eq:upperbound_dpp}, and the expressions of $\zeta$ and $\chi_k(t)$ are omitted due to the lack of space but follow a similar approach as in \cite{Merluzzi2020URLLC,MerluzziDMEC}.}
%\begin{proof}
%See Appendix \ref{appendix}
%\end{proof}
To obtain \eqref{eq:upperbound_dpp}, note first that 
%For a generic virtual queue $X(t)$ evolving as $ X(t+1)=\max(0,\ X(t)+y(t+1)-\bar{y})$, we have from
\cite[p. 59]{Neely10}
%Notice that, from \cite[p. 59]{Neely10} we have:
%\begin{align} \label{eq:init_up}
%    \frac{X_k(t+1)^2-X_k(t)^2}{2}\leq \frac{((t+1)-\bar{y})^2}{2}\\\nonumber
%    +X(t)y(t+1)-X(t)\bar{y}.
%\end{align}
\begin{align} \label{eq:init_up}
    \frac{Z_k(t+1)^2 - Z_k(t)^2}{2} \leq \frac{\parenthesis{Q_k^{\rm tot}(t+1)-Q_k^{\rm avg}}^2}{2}\\
    + Z_k(t)
    \parenthesis{Q_k^{\rm tot}(t+1)- Q_k^{\rm avg}}. \nonumber
\end{align}
%Then %the upper bound
%\eqref{eq:upperbound_dpp} is obtained by first %i) 
Then, apply \eqref{eq:init_up} and the following two inequalities to the $Z_k(t)$'s in $\Delta_p(\mathbf{Z}(t))$: %drift-plus-penalty \eqref{drift_plus_penalty};
%then, %ii)
%by applying the following: %inequality:
$(x+y)^2\leq 2x^2+2y^2, ~\forall~x,y$; %finally, by noting that %iii)
$\max(0,Q-b)+A)^2\leq Q^2+A^2+b^2+2Q(A-b), ~\forall A,b\geq0$.
Full derivations and expressions of $\zeta$ and $\chi_k(t)$ are omitted due to the lack of space, but follow a similar approach as in \cite{Merluzzi2020URLLC,MerluzziDMEC}.
%\textcolor{red}{In the Appendix, we briefly explain how to derive \eqref{eq:upperbound_dpp} from \eqref{drift_plus_penalty}.}
Now, according to the concept of \emph{greedy optimization of a conditional expectation} \cite{Neely10}, %our 
to obtain an optimal policy it is sufficient to minimize %the upper bound of \eqref{drift_plus_penalty} 
\eqref{eq:upperbound_dpp} \emph{in a slot-by-slot fashion}, only based on the observation of instantaneous queue lengths, radio channels, and data generation at the UEs.
%It can be mathematically proved that this is fully  equivalent to optimally solve the following %This is done through the solution of two sub-problems 
The decomposition into two sub-problems is straightforward %Proof follows by first observing that
because radio and computing optimization variables are decoupled in \eqref{eq:upperbound_dpp} and can be treated independently.
\end{proof}

% the first one to find the optimal association variables, and the second one for the optimal CPU scheduling at the ES. The two sub-problems are formulated as follows.}\\
%This can be divided into two sub-problems $\mathcal{P}_1$ and $\mathcal{P}_2$.
%\textcolor{red}{From here we can just say: Minimizing the upper bound of problem \eqref{eq:upperbound_dpp} is equivalent to solving the following two sub-problems, then refer to your paper D-MEC to solve $\mathcal{P}_2$ (with possibly some hints in the appendix where we can introduce lemma 2) and point out why $\mathcal{P}_1$ is more complex and require more advance solution and then we continue with the next sections?? }\textcolor{blue}{Yes, I will write something but that's the idea.} \textcolor{magenta}{Nicola: we should write this as a "Proposition", so that it's easily referable and well visible. Proof can be shortened into a few lines, with references to Neely and papers where we used a similar approach.}

Problem \eqref{sub-pb-computation} can be efficiently and optimally solved using a fast iterative algorithm as in \cite{MerluzziDMEC}, which requires at most $K\times|\mathcal{F}|$ iterations. %, with $|\mathcal{F}|$ the cardinality of $\mathcal{F}$. 
However, \eqref{sub-pb-radio} is more complex as it is non-convex and NP-hard \cite{sana2020UA}. Therefore, we propose %to use 
a MARL strategy, where UEs, modeled as autonomous agents, learn to offload tasks over multiple episodes of random deployments %, in order 
to maximize a long-term $\gamma$-discounted reward $\sum_{\tau=t+1}^{T} \gamma^{\tau-t-1}r(\tau)$, where  $r(t) = -G_2(t)$ is the common reward perceived by each UE at time $t$ and $T$ is the length of an episode. %However, 
From a Lyapunov optimization perspective, the long-term goal (minimization of the long-term average energy) is guaranteed when \eqref{sub-pb-radio} is solved optimally slot by slot. %Accordingly, 
Here, this is achieved by myopically maximizing the instantaneous reward instead of the long-term reward, i.e., by 
setting $\gamma = 0$.

\begin{remark}
%As defined, 
During an episode, $r(t)$ can drop to $-\infty$ due to the presence of queues in the expression of %the objective
$G_2(t)$, which is not bounded. To solve this problem, note that in a feasible scenario, the queues growing to infinity result from UEs deciding systematically to not offload (which is a wrong policy). Hence, we define two clipping values  $Q_k^{\rm clip}= (1+\epsilon_1) Q_k^{\rm avg}$ and $Z_k^{\rm clip} = (1+\epsilon_2) (Q_k^{\rm avg})^2$, parameterized by $\epsilon_1$, $\epsilon_2$ and considered as the maximum tolerable value of physical and virtual queues respectively, above which an episode terminates with a failure. In this way, we improve the learning convergence, as UEs are quickly notified of their failure. %\textcolor{red}{In Appendix \ref{appendix}, we present our approach to derive suitable values of $Q_k^{\rm clip}$ and $Z_k^{\rm clip}$ based on simulation parameters.}
\end{remark}
\vspace{-0.3cm}
\section{Proposed Lyapunov-aided %multi-agent reinforcement learning 
MARL for mobile edge computing}
\vspace{-0.1cm}
\begin{figure}
    \centering
    \includegraphics[scale=0.75]{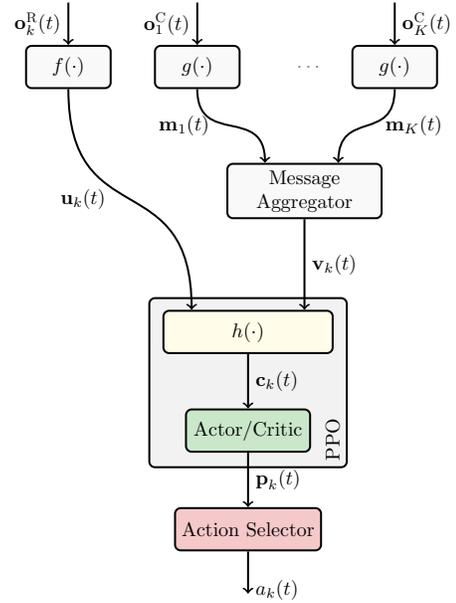}
    \caption{MEC Offloading policy network architecture.}
    \label{fig:pna-archi}
\end{figure}
%\subsection{General framework}
To solve the radio resource allocation problem \eqref{sub-pb-radio}, we propose a MARL framework. \fig{fig:pna-archi} describes the proposed architecture, made of encoding functions $f(\cdot)$, $g(\cdot)$, and $h(\cdot)$ for features extraction, a message aggregator, an actor-critic module for policy training, and an action selector for task offloading decisions. %In this architecture, 
UEs share the same policy architecture and the same parameters. However, this does not preclude UEs from taking different decisions, due to their different environment observations in the same slot. In contrast, sharing parameters limits complexity (as there is only one common policy to learn), enables efficient training, and helps to improve convergence. Inspired by our previous work \cite{sana2020UA}, let $\obsRadio{k}$ denote the set of ``radio observations'' of UE $k$:
\vspace{-0.1cm}
\begin{align}
    \begin{split}
        \obsRadio{k} {}={}& \bigg\{a_k(t-1), R_{k,a_k(t-1)}, R(t-1), \\
        &{} \mathrm{ACK}_k, \left\{\mathrm{RSS}_{k,n}\right\}_{n \in \mathcal{A}_k}, \left\{\vartheta_{k,n}\right\}_{n \in \mathcal{A}_k} \bigg\}.
    \end{split}
\end{align}
%\vspace{-0.1cm}
$a_k(t-1)\in\mathcal{A}_k$ denotes UE $k$'s action (i.e., connection request to an AP) at time $t-1$, $R_{k,a_k(t-1)}$ is the perceived rate, $R(t-1)$ the total network sum-rate, and $\mathrm{ACK}_k$ the received connection acknowledgment signal. $\left\{\mathrm{RSS}_{k,n}\right\}_{n \in \mathcal{A}_k}, \left\{\vartheta_{k,n}\right\}_{n \in \mathcal{A}_k}$ indicate the received signal strength and corresponding angles of arrival (AoA) from UE $k$ to AP $n$. %, respectively.
Similarly, %we denote with 
$\obsComp{k}$ represents %the set of 
``MEC observations'', %i.e., observation 
related to task offloading:
\begin{align}
    \begin{split}
        \obsComp{k} {}={}& \bigg\{(x_k, y_k), f_k(t), Q_k^l(t), Q_k^{\rm s}(t), Z_k(t) \bigg\},
    \end{split}
\end{align}
where $Q_k^l(t)$, $Q_k^{\rm s}(t)$, $Z_k(t)$ are the queues defined above, $(x_k, y_k)$ %is the estimated localization of UE $k$
are UE $k$'s geographical coordinates, and $f_k(t)$ its allocated CPU frequency at the ES. In our framework, after observing $\obsRadio{k}$, UE $k$ builds its local state encoding $\mathrmbold{u}_k$, which represents its ``perception'' of the radio environment, using an encoding function $f(\cdot)$, e.g., a neural network. Then, based on %its MEC observations 
%$\obsComp{k}$ and after aggregating messages 
the aggregated MEC observations of its whole neighborhood, it constructs an encoding vector $\mathrmbold{v}_k$, which characterizes its perception of the network from a computation viewpoint. UE $k$ then builds its overall context encoding vector $\mathrmbold{c}_k$ to represent its global understanding of the environment, using an encoding function $h(\cdot)$, e.g., a concatenation operator or a neural network. %Finally,
For each UE, the goal of the MARL framework is to learn an association policy $\pi_{\params}$ %parameterized by 
with learnable parameters $\params$, where $\pi_{\params}(a_k(t)| \mathrmbold{o}_k(t)) = p_{a_k(t), k}$ indicates the probability of taking action $a_k(t)$ after observing $\mathrmbold{o}_k(t) = \{\obsRadio{k}, \obsComp{k}\}$. Note that %the probability vector 
$\mathrmbold{p}_k(t) = [p_{0,k}, \ldots, p_{N,k}] \in [0,1]^{N+1}$, from which the action $a_k(t)$ of the UE will be sampled, is such that $\sum_{n\in\mathcal{A}} p_{n,k} = 1$ and $p_{n,k} = 0$ for all $n \not \in \mathcal{A}_k$.
%This optimization is done end-to-end by using \emph{proximal policy optimization} (PPO) and an actor-critic framework \cite{schulman2017ppo}.

\subsection{Message passing service}
Let $\weights_k, \weights_q, \text{and~} \weights_\nu: \mathbb{R}^6 \times \mathbb{R}^m$ be learnable weights, describing the set of parameters of the encoding function $g(\cdot)$, which we later refer to as the \emph{message generator}. For each UE $l$, let $\mathrmbold{k}_l =  \weights_k^T \obsComp{l}$, $\mathrmbold{q}_l =  \weights_q^T \obsComp{l}$, $\mathrmbold{\nu}_l =  \weights_{\nu}^T \obsComp{l}$, and $\mathrmbold{m}_l = \{ \mathrmbold{q}_l, \mathrmbold{\nu}_l \}$ %, where $\mathrmbold{k}_l$, $\mathrmbold{q}_l$, $\mathrmbold{\nu}_l$ denote 
be the \emph{key}, the \emph{query}, the \emph{value}, and the \emph{message} associated with UE $l$. %, respectively. 
Then, each UE $k$, after aggregating the messages from its neighbors $\mathcal{N}_k$, computes its encoding vector %$\mathrmbold{v}_k$: % as follows:
%\begin{equation}\label{eq:dot-att-value}
%    \mathrmbold{v}_k =  \sum_{l\in\mathcal{N}_k} \alpha_{l,k} \mathrmbold{\nu}_l,
%\end{equation}
$\mathrmbold{v}_k =  \sum_{l\in\mathcal{N}_k} \alpha_{l,k} \mathrmbold{\nu}_l$,
where the score $\alpha_{l,k}$ %is a score, which
represents the interaction between UEs $l$ and $k$ (in achieving the underlying optimization goal). This score is calculated using dot-product attention mechanism \cite{vaswani2017attention}:
%\begin{equation} \label{scoreCompute}
%    \alpha_{l,k} = \softmax\left(\left[\frac{\mathrmbold{q}_l \mathrmbold{k}_k^T}{\sqrt{n}}\right]_{l\in\mathcal{N}_k}\right).
%\end{equation}
$\alpha_{l,k} = \softmax\left(\left[\frac{\mathrmbold{q}_l \mathrmbold{k}_k^T}{\sqrt{m}}\right]_{l\in\mathcal{N}_k}\right)$.
Here, $\softmax(\cdot)$ denotes the normalized exponential function. %It is noteworthy that %equation \eqref{eq:dot-att-value}
Note that computing $\mathrmbold{v}_k$ only involves the queries and the values from others UEs in $\mathcal{N}_k$ and not their keys, which are UE-specific and do not need to be transmitted. %In addition, %it is worth noting that 
Such a message passing service enables the \emph{scalability} and the \emph{transferability} of the learned policy, which is optimized for all possible UE deployments, in sharp contrast with \cite{Bi2020}, which requires fixed UEs. In other words, in our framework, a change in the number or position of UEs in the network %and/or their position 
does not require a new policy training and does not impact the architecture of the policy network. Only the number of exchanged messages varies, depending on the variation of a UE's neighborhood. Both, the input variables and the number of neurons of the encoding functions remain fixed. %a change in the input variables nor in the number of neurons of the encoding functions.
This enables \emph{curriculum learning}, where a policy obtained from e.g.~6 UEs can be leveraged as a starting point to train another policy for $K>6$ UEs. Finally, all %policy network 
the encoding functions, including the message generator, are optimized through end-to-end %learning procedure using
\emph{proximal policy optimization} (PPO) and an actor-critic framework \cite{schulman2017ppo}.

\section{Numerical Results}
\begin{comment}
\begin{table}[!t]
    \centering
    \caption{Simulations parameters \textcolor{red}{todo} \protect\cite{sana2020UA}}
    \label{simu-params}
    \scalebox{0.9}{
    \begin{tabular}{|l||c|c|c|}
       \hline
        & UEs & APs & ES\\
       \hline
       \hline
        Parameters & \multicolumn{3}{c|}{Values} \\
       \hline
        Carrier frequency (GHz) & \multicolumn{2}{c|}{28} &\\
        \hline
        Bandwidth (MHz) & 10 & 200 &\\
        \hline
        Thermal noise (dBm/Hz) & \multicolumn{2}{c|}{-174} &\\
        \hline
        Shadowing variance &\multicolumn{2}{c|}{12 dB}  & \\
       \hline
        $p^{\rm off}$, $p^{\rm on}$, $p^{\rm tx}$ (mW)& 346, 900, 100 & 292, 251, 2200 & $10^4$, $2\times 10^4$, - \\
       \hline
        Antenna gain & \cite{sana2020UA}, (Diagram 2) & & \\ 
       \hline
        Radius, $R_0$ & & 50 m &\\
       \hline
        Back-lobe gain & & -20 dB &\\
       \hline
        Path-loss exponent & \multicolumn{2}{c|}{2.5} & \\
        %\hline
        %\textcolor{red}{Path-loss constant$^{(1)}$} & & \\
       \hline
        Inter-cell distance & & $1.2\times R_0$ &\\
       \hline
    \end{tabular}}
\end{table}

\begin{figure}
    \centering
    \includegraphics[scale=0.8]{images/antennadiag.pdf}
    \caption{Simulated antenna radiation pattern at 28 GHz \cite{mailloux2017phased}}
    \label{fig:antennadiag}
\end{figure}
\end{comment}

In this section, we assess the effectiveness of the proposed framework in a network of 3 APs\footnote{The coverage range is $R_0=50$ m and the inter-cell distance is $1.2\times R_0.$} operating at 28-GHz mmWave frequencies and for $K\in\{6, 9, 12, 15\}$ UEs. We use $p^{\rm u, off} = 0.346$ W, $p^{\rm u, on}=0.9$ W, $p^{\rm a, off} = 0.278$ W, $p^{\rm a, on}=2.2$ W, $p^{\rm s, off} = 10$ W, $p^{\rm s, on}=20$ W. Each UE transmits with power $p^{\rm u, tx}(t) = \minof{p_k^{\rm tg}(t)}{p_{\max}}$ over a bandwidth $B=10$ MHz, where $p_k^{\rm tg}(t)$ is the power to meet a predefined target SNR of $15$ dB and $p_{\max}=0.1$ W. Each slot lasts $10$ ms and we set $\beta=0.1$, $N_n=15$, $\kappa=10^{-27}$, $J_k=10^{-3}$, $S_k=1500$ bits $\forall k$ and $\mathcal{F}=\{0,0.1,\ldots,1\}\times %f_{\max}$, where $f_{\max}=
10^9 \text{~cycle/s}$. UEs' data generation rate follows a Poisson distribution with mean $D_k=50\times S_k$ bits $\forall k$. %\textcolor{blue}{We assume that both TX/RX uses the same simulated antenna radiation pattern of \fig{fig:antennadiag}. The interference in this case, results from the overlapping of different beams serving different UEs.} 
Additional parameters, including pathloss and antenna diagrams, can be found in \cite[Table I]{sana2020UA}. %Table \ref{simu-params} summarizes the simulations parameters. 
In our setup, all encoding functions are composed of one multi-layer perceptron (MLPs) of $m=128$ neurons with a rectifier linear unit (ReLu) activation. Both the actor and the critic module comprise $2m$ neurons and we empirically set the learning rate to $10^{-4}$, $\epsilon_1=10$ and $\epsilon_2=0$. To foster the learned policy and enable better generalization, during training, we consider random CPU scheduling\footnote{We randomly select $f_c(t) \in \mathcal{F}$ and allocate $\omega_k f_c(t)$ to each UE $k$ such that $\sum_k \omega_k = 1$, where $\{\omega_k\}_{k}$ follow a symmetric Dirichlet distribution.}. This is possible since \eqref{sub-pb-computation} and \eqref{sub-pb-radio} are completely decoupled, therefore, the policy learned to solve \eqref{sub-pb-radio} must be independent of the ES frequency allocation. We compare our solution, labeled L2OFF (Learning to Offload) in \figs{fig:perfEnergy}{fig:perfComparison}, to two benchmarks:
\begin{itemize}
    \item Exhaustive search: at each slot, we perform an exhaustive search over all possible solutions of \eqref{sub-pb-radio}.
    
    \item Max-SNR: %here, a 
    each UE is associated with a Bernoulli random variable with probability $p$ of being in active state (which models %corresponding to 
    the average duty cycle of UEs). Then, at each $t$, an active UE gets associated with the AP providing the maximum signal-to-noise ratio (SNR).
\end{itemize}
All results are averaged over 200 random deployments of UEs.
\begin{figure}
    \centering
    \includegraphics[width=\columnwidth]{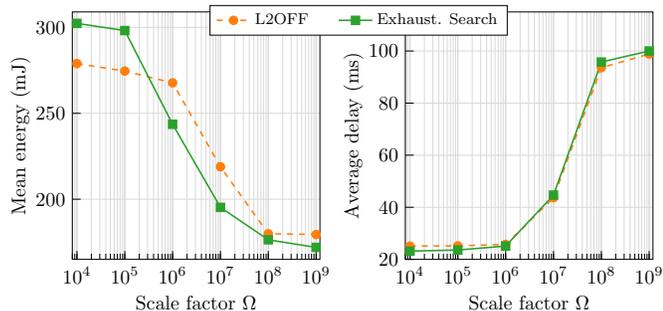}
    \caption{Energy-delay trade-off \wrt $\Omega$ for $K=6$ UEs and for a fixed delay constraint of $100$ ms.}
    \label{fig:perfEnergy}
\end{figure}

\vspace{-0.3cm}
\noindent\textbf{Energy-delay trade-off vs. $\mathrmbold{\Omega}$}. Here, we evaluate the performance of our proposed framework for different values of $\Omega$ (cf. \eqref{drift_plus_penalty}), and compare the results to the performance obtained via exhaustive search in~\fig{fig:perfEnergy}. First, we observe the results suggested by the theory: when %by letting 
$\Omega$ increases, optimally solving \eqref{sub-pb-computation} and \eqref{sub-pb-radio} leads to a lower energy consumption. Meanwhile, the average delay increases and caps to 100ms, which is the fixed delay constraint \eqref{eq:C1}. Interestingly, the proposed scheme exhibits performance close to exhaustive search approach (for $\Omega=10^9$), reaching up to $96.5\%$ of its performance, for the same delay.

\vspace{0.1cm}
\noindent
\textbf{Performance Comparison}. To fairly compare the proposed framework with the heuristic based on Max-SNR, we first determine exhaustively the optimal lowest duty cycle that enables the Max-SNR algorithm to guarantee an average delay constraint of $100$ ms. Then, comparison is made for the same delay in \fig{fig:perfComparison}. %\footnote{For complexity reasons, results for $K\in\{12,15\}$ UEs cannot be obtained for the exhaustive search.}. %We observe that our framework can achieve up to $96.5\%$ of the optimal solution obtained through exhaustive search. 
We can notice how, even by optimally computing the duty cycle for the Max-SNR algorithm, our solution still outperforms, reducing the energy by $10\%$ for 15 UEs compared to Max-SNR solution, as we consider interference, and intelligently orchestrate UEs. Moreover, under the same delay constraint, with our strategy, the network consumes $246$ mJ in average for $15$ UEs, whereas for the same energy consumption, the Max-SNR can only serve $12$ UEs.
\begin{comment}
\begin{figure}
    \centering
    \includegraphics[width=\columnwidth]{images/perf-duty-cycle.pdf}
    \caption{Performance comparison, $\Omega$}
    \label{fig:perfDutyCycle}
\end{figure}
\end{comment}

%\vspace{-0.1cm}
\section{Conclusion}
%\vspace{-0.1cm}
In this work, we proposed %present
a novel approach for delay constrained energy-efficient computation offloading services in dense mmWave networks impaired by interference. We first formulated the computation offloading as a long-term optimization.
%the problem in terms of long-term energy minimization and duty cycle optimization,
%a long-term computation offloading optimization problem, 
%Starting from a long-term optimization problem, thanks to Lyapunov optimization, we proposed a low-complexity solution based on multi-agent reinforcement learning to solve the UE-AP association, and an efficient algorithm to solve the CPU scheduling at the edge server. 
Then, we applied Lyapunov optimization tools to split the problem into a CPU scheduling problem and a UE-AP association problem. While the first one is easily solvable via an efficient iterative algorithm, we solved the second one thanks to multi-agent reinforcement learning with a \emph{distributed and transferable} policy.
The proposed solution %can
reaches up to $96.5\%$ of the optimal solution obtained via exhaustive search and can reduce energy consumption up to $10\%$ compared to a heuristic approach based on SNR maximization. %With respect to the Max-SNR heuristic, our method is also capable of serving more UEs with the same energy consumption and average delay.

%\vspace{-0.1cm}
%\section*{Appendix - Derivation of %\eqref{eq:upperbound_dpp}}
%%\label{appendix}
%%\noindent
%%\begin{proof}[Proof Eq. \eqref{eq:upperbound_dpp}]
%Note that 
%%For a generic virtual queue $X(t)$ evolving as $ X(t+1)=\max(0,\ X(t)+y(t+1)-\bar{y})$, we have from
%\cite[p. 59]{Neely10}:
%%Notice that, from \cite[p. 59]{Neely10} we have:
%%\begin{align} \label{eq:init_up}
%%    \frac{X_k(t+1)^2-X_k(t)^2}{2}\leq \frac{((t+1)-\bar{y})^2}{2}\\\nonumber
%%    +X(t)y(t+1)-X(t)\bar{y}.
%%\end{align}
%\begin{align} \label{eq:init_up}
%    \frac{Z_k(t+1)^2 - Z_k(t)^2}{2} \leq \frac{\parenthesis{Q_k^{\rm tot}(t+1)-Q_k^{\rm avg}}^2}{2}\\
%    + Z_k(t)
%    \parenthesis{Q_k^{\rm tot}(t+1)- Q_k^{\rm avg}}. \nonumber
%\end{align}
%%Then %the upper bound
%\eqref{eq:upperbound_dpp} is obtained by first %i) 
%applying \eqref{eq:init_up} to the \textcolor{red}{$Z_k(t)$'s in the} drift-plus-penalty \eqref{drift_plus_penalty}; then, %ii)
%by applying the following: %inequality:
%$(x+y)^2\leq 2x^2+2y^2, ~\forall~x,y$; finally, by noting that %iii)
%$\max(0,Q-b)+A)^2\leq Q^2+A^2+b^2+2Q(A-b), ~\forall A,b\geq0$.
%\textcolor{red}{Full derivations and expressions of $\zeta$ and $\chi_k(t)$ are omitted due to the lack of space, but follow a similar approach as in \cite{Merluzzi2020URLLC,MerluzziDMEC}.}

%\begin{align*}\label{eq:upper_Z1}
%    \frac{Z_k(t+1)^2 - Z_k(t)^2}{2} \leq \parenthesis{Q_k^l(t+1)}^2+ \parenthesis{Q_k^{\rm s}(t+1)}^2\\ + \frac{1}{2}\left(Q_k^{\rm avg}\right)^2 
%    + Z_k(t) \parenthesis{Q_k^{\rm tot}(t+1)- Q_k^{\rm avg}}.
%\end{align*}

%\end{proof}

\begin{figure}[!t]
    \centering
    \includegraphics[width=0.97\columnwidth]{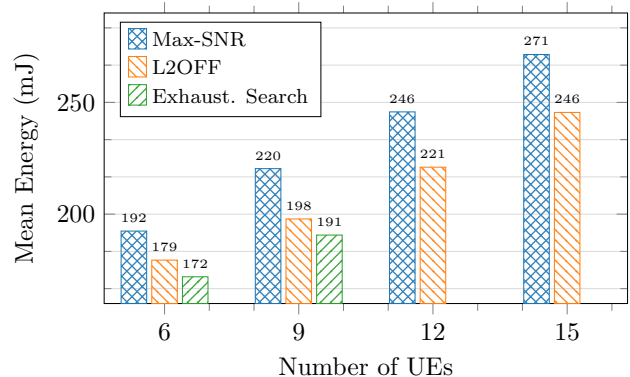}
    \caption{Average energy for a fixed average delay of $100$ ms. Due to complexity, results for $K\in\{12,15\}$ UEs cannot be obtained for the exhaustive search.}
    \label{fig:perfComparison}
\end{figure}

%\vspace{-0.1cm}
\bibliographystyle{IEEEtran}
\bibliography{IEEEabrv, Biblio_Mattia, biblio}
\end{document}